\documentclass[11pt,fleqn]{article}

\usepackage{times}
\usepackage[german,english]{babel}

\usepackage{authblk}

\usepackage{amssymb}
\usepackage{xspace}
\usepackage{color}
\usepackage{subfigure}
\usepackage{graphicx}
\usepackage{amsmath}
\usepackage{amssymb}
\usepackage{amsfonts}
\usepackage{float}
\usepackage{url}
\usepackage{latexsym}
\usepackage{pifont}
\usepackage{stmaryrd}
\usepackage{cleveref}
\usepackage{todonotes}
\usepackage{relsize}

\newtheorem{theorem}{Theorem}[section]
\newtheorem{lemma}[theorem]{Lemma}

\newtheorem{example}[theorem]{Example}
\newtheorem{definition}[theorem]{Definition}
\newtheorem{proposition}[theorem]{Proposition}

\usepackage{tikz}
\usetikzlibrary{arrows,backgrounds}
\usetikzlibrary{backgrounds}
\tikzstyle{arg}=[circle,draw,thick,minimum size=8mm]
\tikzstyle{targ}=[circle,draw=\truecolor,thick,minimum size=8mm]
\tikzstyle{farg}=[circle,draw=\falsecolor,thick,minimum size=8mm]
\tikzstyle{att}=[-stealth,thick]
\tikzstyle{link}=[->,thick]

\def\truecolor{black}
\def\falsecolor{black}

\newcommand{\tvt}{\textcolor{\truecolor}{\mathbf{t}}}
\newcommand{\tvf}{\textcolor{\falsecolor}{\mathbf{f}}}
\newcommand{\tvu}{\mathbf{u}}
\newcommand{\tv}[1]{\mathbf{#1}}
\newcommand{\gpi}{\mathit{gpi}}

\newcommand{\define}[1]{\emph{#1}}
\newcommand{\set}[1]{\ensuremath{\left\{#1\right\}}}
\newcommand{\guard}{\ \middle\vert\ }

\newcommand{\parents}[1]{\mathit{par}(#1)}
\newcommand{\ileq}{\ensuremath{\leq_i}}

\newcommand{\icap}{\textstyle\bigsqcap_i}
\newcommand{\icup}{\textstyle\bigsqcup_i}
\newcommand{\ilt}{<_i}

\newcommand{\Vu}{V_\tvu}
\newcommand{\vu}{v_\tvu}
\newcommand{\two}{\set{\tvt,\tvf}}
\newcommand{\three}{\set{\tvt,\tvf,\tvu}}
\newcommand{\INT}{\textit{INT}}
\newcommand{\lfp}{\mathit{lfp}}

\newcommand{\adf}{ADF\xspace}
\newcommand{\adfs}{ADFs\xspace}

\newcommand{\adm}{\textit{adm}}
\newcommand{\com}{\textit{com}}
\newcommand{\prf}{\textit{prf}}
\newcommand{\md}{\textit{mod}}

\newcommand{\ac}{\varphi}
\newcommand{\Ac}{\Phi}

\newcommand{\unit}{\lbrack 0,1\rbrack}

\newcommand{\ie}{i.\,e.,\xspace}				\newcommand{\eg}{e.\,g.,\xspace}				

\newcommand{\wrt}{with respect to }

\newcommand{\NP}{\textsf{NP}}
\newcommand{\coNP}{\textsf{coNP}}

\newcommand{\SigmaP}[1]{\Sigma^{p}_{#1}}
\newcommand{\PiP}[1]{\Pi^{p}_{#1}}

\newcommand{\naturals}{\ensuremath{\mathbb{N}}}

\newcommand{\myparagraph}[1]{\vskip3pt\noindent\textbf{#1}\,}

\usepackage{comment}

\newenvironment{proof}   {\begin{trivlist}\item[]\textbf{Proof.}}   {\hfill$\Box$\end{trivlist}}

\begin{document}

\title{Weighted Abstract Dialectical Frameworks: \\Extended and Revised Report\thanks{This is an extended and corrected version of the paper \emph{Weighted Abstract Dialectical Frameworks} published in the Proceedings of the 32nd AAAI Conference on Artificial Intelligence (AAAI 2018). The conference paper was based on a definition of completions suitable for flat information orderings only. In particular, Prop.~1 in the AAAI paper only holds under the restriction to flat orderings. We thank Bart Bogaerts for pointing this out to us. The new definition of completions used in this report (Def.~\ref{def:completion}) does not suffer from these shortcomings.}
}

\author[1]{Gerhard~Brewka}
\author[2]{J{\"o}rg~P{\"u}hrer}
\author[1]{Hannes~Strass}
\author[2]{\authorcr Johannes~P.~Wallner}
\author[2]{Stefan~Woltran}
\affil[1]{Leipzig University, Computer Science Institute, Leipzig, Germany.\newline
\textit{brewka@informatik.uni-leipzig.de, hannes.strass@googlemail.com}}
\affil[2]{TU Wien, Institute of Logic and Computation, Vienna, Austria.\newline
\textit{jpuehrer,wallner,woltran@dbai.tuwien.ac.at}}

\date{}
\maketitle

\begin{abstract}
Abstract Dialectical Frameworks (ADFs) generalize Dung's argumentation frameworks allowing various relationships among arguments to be expressed in a systematic way. We further generalize ADFs so as to accommodate arbitrary acceptance degrees for the arguments.
This makes
ADFs applicable in domains where both the initial status of arguments
and their relationship are only insufficiently specified by Boolean functions.
We define all standard ADF semantics for the weighted case, including grounded, preferred and stable semantics.
We  illustrate our approach using acceptance degrees from the unit interval and show how other valuation structures can be integrated. In each case it is sufficient to specify how the generalized acceptance conditions are represented by formulas, and to specify the information ordering underlying the characteristic ADF operator.
We also present complexity results for
problems related to weighted ADFs.
\end{abstract}

\section{Introduction}\label{sec:intro}
Computational models of argumentation are a highly active area of current research. The field has two main subareas, namely logic-based argumentation and abstract argumentation. The former studies the structure of arguments, how they can be constructed from a given formal knowledge base, and how they logically interact with each other. The latter, in contrast, assumes a given set of abstract arguments together with specific relations among them. The focus is on evaluating the arguments based on their interactions with one another. This evaluation typically uses a specific semantics, thus identifying subsets of the available arguments satisfying intended properties so that the chosen set arguably can be viewed as representing a coherent world view.

In the abstract approach, Dung's argumentation frameworks (AFs) \cite{Dung1995} and their associated semantics are widely used. In a nutshell, an AF is a directed graph with each vertex being an abstract argument and each directed edge corresponding to an attack from one argument to another. These attacks are then resolved using appropriate semantics. The semantics are typically based on two important concepts, namely conflict-freeness and admissibility. The former states that if there is a conflict between two arguments, i.e.\ one argument attacks the other, then the two cannot be jointly accepted. The latter specifies that every set of accepted arguments must defend itself against attacks. A variety of semantics has been defined, ranging from Dung's original complete, preferred, stable, and grounded semantics to the more recent ideal and {cf2} semantics. The different semantics reflect different intuitions about what ``coherent world view'' means in this context, see e.g.\ \cite{BaroniCG2011} for an overview.

Despite their popularity, there have been various attempts to generalize AFs as many researchers felt a need to cover additional relevant relationships among arguments (see e.g.\ the work of \cite{CayrolL09}). One of the most systematic and flexible outcomes of this research are
abstract dialectical frameworks (ADFs) \cite{BrewkaW2010,BrewkaESWW2013,ifcolog}.
ADFs allow for arbitrary relationships among arguments. In particular, arguments can not only attack each other, they also may provide support for other arguments and interact in various complex ways. This is achieved by adding explicit acceptance conditions to the arguments which are most naturally expressed in terms of a propositional formula (with atoms referring to parent arguments).
This way, it is possible to specify individually for a particular argument, say, under what conditions the available supporting arguments outweigh the counterarguments. Meanwhile various applications
of ADFs have been presented, for instance in
legal reasoning \cite{Al-AbdulkarimAB14,Al-Abdulkarim2016} and text exploration \cite{CabrioV16}.
A mobile argumentation app based on ADF techniques was developed by \cite{Puehrer17}.

The operator-based semantics of ADFs can be traced back to work
on
approximation fixpoint theory (AFT)
\cite{denecker00approximations,denecker03uniformsemantic,denecker04ultimate},
an algebraic framework for studying semantics of knowledge representation formalisms.
We refer to the work of \cite{Strass2013c} for a detailed analysis of the relationship between ADFs and AFT. The presentation of our approach in this paper does not assume specific background knowledge in AFT.

The motivation for the work presented here is as follows. The definition of the various ADF semantics is based on an analysis in terms of partial two-valued (or, equivalently, three-valued) interpretations. The output provided by ADFs (and AFs, for that matter) is thus restricted to three options: an argument either is true (accepted) in an intended interpretation, or it is false (rejected), or its value is unknown. However, many situations in argumentation call for more fine-grained distinctions (see, e.g.\ \cite{AlsinetABFMP17} for an application of weighted argumentation in the Twitter domain). For instance, it is sometimes natural to assume numerical acceptance degrees, say, taken from the unit interval, and to explore the effect of these degrees on other arguments. The availability of such acceptance degrees allows for new, interesting types of queries to be asked. For instance, under a given semantics (stable, preferred, complete, \ldots), we may want to know  whether the value of a particular argument $s$ is above/below a certain threshold in some or all interpretations of the required type.
It also may be useful to be able to distinguish among a finite number of acceptance degrees, say strong accept, weak accept, neutral, weak reject and strong reject. Or it may even be useful to operate on intervals of acceptance degrees.

The goal of this paper is to show how the ADF approach (and thus AFs) can accommodate such acceptance degrees. To put it differently, we aim to bridge two rich research areas,
multi-valued logics on the one hand and computational models of argumentation on the other.

We start with the necessary ADF background in \Cref{sec:background}.
We then introduce our general framework for weighted ADFs in \Cref{sec:framework}.
\Cref{sec:unit}
focuses on ADFs with acceptance degrees in the unit interval. \Cref{sec:alternative} applies the same idea to three other valuation structures. Complexity results for various problems related to weighted ADFs are presented in \Cref{sec:complexity}. \Cref{sec:conclusions}
discusses  related work and concludes.

\section{Background}\label{sec:background}
An \adf is a directed node-labelled graph $(S,L,C)$ whose nodes represent statements. The links in $L$ represent dependencies: the status of a node $s$ only depends on the status of its parents (denoted $\parents{s}$), that is, the nodes with a direct link to $s$.
In addition, each node $s$ is labelled by an associated acceptance condition $C_s$ specifying the conditions under which $s$ is acceptable, whence \mbox{$C=\set{C_s}_{s\in S}$}.
Formally, the acceptance condition $C_s$ of node $s$ with parents $\parents{s}$ is a function
\mbox{$C_s: (\parents{s} \to \{\tvt,\tvf\}) \to \{\tvt,\tvf\}$}.
It is convenient to represent the acceptance conditions as a collection $\Ac = \{\ac_s\}_{s \in S}$ of propositional formulas (using atoms from $\parents{s}$ and connectives $\land$, $\lor$, $\neg$).
Then, for any interpretation $w:\parents{s}\to\set{\tvt,\tvf}$, we have $C_s(w)=w(\ac_s)$, that is, the acceptance condition $C_s$ evaluates $w$ just like $w$ evaluates $\ac_s$.
This leads to the logical representation of \adfs we will frequently use, where an \adf is a pair $(S,\Ac)$ with the set of links $L$ implicitly given as $(a, b) \in L$ iff $a$ appears in $\ac_b$.

Semantics assign to ADFs a collection of partial two-valued interpretations, i.e.\ mappings of the statements to values $\{\tvt,\tvf,\tvu\}$ where $\tvu$ indicates that the value is undefined. Mathematically such interpretations are equivalent to 3-valued interpretations, but for the purposes of this paper it is beneficial to view them (interchangeably) also as partial interpretations.
The three values are partially ordered by $\ileq$ according to their information content:
$\ileq$ is the $\subseteq$-least partial order containing  \mbox{$\tvu \ileq \tvt$} and \mbox{$\tvu \ileq \tvf$}.
As usual we write $v_1 \ilt v_2$ whenever $v_1 \ileq v_2$ and not $v_2 \ileq v_1$.
The information ordering $\ileq$ extends in a straightforward way to partial interpretations $v_1,v_2$ over $S$ in that
\mbox{$v_1 \ileq v_2$} if and only if \mbox{$v_1(s) \ileq v_2(s)$} for all \mbox{$s\in S$}.

A partial
interpretation $v$ is total if all statements are mapped to $\tvt$ or $\tvf$.
For interpretations $v$ and $w$, we say that $w$ extends $v$ iff \mbox{$v \ileq w$}.
We denote by $[v]_2$ the set of all completions of $v$, i.e.\ total interpretations that extend~$v$.

For an ADF
$D=(S,L,C)$,
statement $s\in S$ and a partial
interpretation $v$, the characteristic operator $\Gamma_D$ is given by
\[
\Gamma_D(v)(s) = \begin{cases}
  \tvt &\text{if }  C_s(w) = \tvt \text{ for all } w\in[v]_2,\\
  \tvf &\text{if }  C_s(w) = \tvf \text{ for all } w\in[v]_2,\\
  \tvu &\text{otherwise.}
\end{cases}
\]

That is, the operator returns an
interpretation mapping a statement $s$
to $\tvt$ (resp.\ $\tvf$) if and only if all two-valued interpretations extending $v$ evaluate $\ac_s$ to true (resp.\ false).
Intuitively, $\Gamma_D$ checks which truth values can be justified based on the information in $v$ and the acceptance conditions.

Given
an \adf $D=(S,L,C)$,
a partial
interpretation $v$ is \emph{grounded} \wrt $D$ if it is the least fixpoint of $\Gamma_D$; it is \emph{admissible} \wrt $D$ if
\mbox{$v \ileq \Gamma_D(v)$}; it
is \emph{complete} \wrt $D$ if
\mbox{$v = \Gamma_D(v)$}; it
is a \emph{model} of $D$ if it is complete and total;
it is \emph{preferred} \wrt $D$ if $v$ is maximally admissible \wrt $\ileq$.
As shown in \cite{BrewkaESWW2013}
these semantics generalize the corresponding notions defined for AFs.
For $\sigma \in \{\adm,\com,\prf\}$, $\sigma(D)$ denotes the set of all
admissible (resp.\ complete, preferred) interpretations \wrt $D$.

\begin{example}
\label{ex:adf}
Given \adf $D$ over $\set{a,b}$ with $\ac_a = a \vee \neg b$, $\ac_b = \neg a$,
and
interpretations
\begin{eqnarray*}
v_1 &=& \{a \mapsto \tvu, b \mapsto \tvu\},\\
v_2 &=& \{a \mapsto \tvt, b \mapsto \tvu\},\\
v_3 &=& \{a \mapsto \tvt, b \mapsto \tvf\},\\
v_4 &=& \{a \mapsto \tvf, b \mapsto \tvt\}.
\end{eqnarray*}
We get
$\adm(D) = \{v_1,v_2,v_3,v_4\}$,
$\com(D) = \{v_1,v_3,v_4\}$
(note that $\Gamma_D(v_2)=v_3$, and thus $v_2 \notin \com(D)$),
and
$\prf(D) = \{v_3,v_4\}$.~$\diamond$
\end{example}

\section{The General Framework}
\label{sec:framework}
In this section we introduce weighted ADFs (wADFs).
More precisely, we introduce a general framework which allows us to define wADFs over a chosen set $V$ of values (acceptance degrees) based on an information ordering $\ileq$ on $V \cup \{\tvu\}$.\footnote{Slightly abusing notation we write $\ileq$ for both the specific ADF ordering $(\three,\ileq)$ and the generic ordering used here.}
\begin{definition}
  A weighted ADF (wADF) over $V$ is a tuple $D=(S,L,C,V,\ileq)$, where
  \begin{itemize}
  \item $S$ is a set (of nodes, statements, arguments; anything one might accept or not),
  \item $L \subseteq S\times S$ is a set of links,
  \item $V$ is a set of truth values with $\tvu \not \in V$,
  \item $C = \set{C_s}_{s\in S}$ is a collection of acceptance conditions over $V$, that is, functions \mbox{$C_s: (\parents{s} \to V) \to V$},
  \item $(\Vu,\ileq)$ -- where \mbox{$\Vu = V\cup\set{\tvu}$} -- forms a complete partial order with least element $\tvu$.
  \end{itemize}
\end{definition}
\noindent As for standard ADFs, the special value $\tvu$ represents an undefined truth value.
As usual, $(\Vu,\ileq)$ forms a complete partial order (CPO) iff:
(1) it has a least element, here \mbox{$\tvu\in\Vu$},
(2) each non-empty subset \mbox{$X\subseteq\Vu$} has a greatest lower bound \mbox{$\icap X\in\Vu$}, and
(3) each ascending chain \mbox{$x_1\ileq x_2\ileq \ldots$} over $\Vu$ has a least upper bound \mbox{$\icup X\in\Vu$}.

ADFs are a special case of wADFs with \mbox{$V=\two$} and the information ordering as defined in the background section. We provide a formal result in Theorem~\ref{thm:generalization} below.

\begin{figure}
\centering
 \begin{tikzpicture}[node distance=6mm]
    \node[draw,rectangle] (a) at (0,0) {paper's status};
    \node[draw,rectangle] (b) at (-2.2,-2.2) {paper's significance};
    \node[draw,rectangle] (c) at (2.2,-2.2) {scientific methodology};
    \path[link] (c) edge[] (a);
    \path[link] (b) edge[] (a);
    \node at (0,-3.4) {(a)};

    \node (v1) at (-3,-5)  {accept};
    \node (v2) at (0,-5)  {borderline};
    \node (v3) at (3,-5)  {reject};
    \node (v4) at (-2.5,-6) {tendency accept};
    \node (v5) at (1.5,-6) {tendency reject};
    \node (no) at (0,-7) {no tendency};
    \node (u) at (0,-7.8) {$\tvu$};
    \draw (u) -- (no);
    \draw (no) -- (v5);
    \draw (no) -- (v4);
    \draw (v5) -- (v3);
    \draw (v5) -- (v2);
    \draw (v4) -- (v1);
    \draw (v4) -- (v2);
    \node at (0,-8.4) {(b)};
 \end{tikzpicture}
 \caption{Example wADF (a) and example $\Vu$ (b)\label{fig:wadf-simple-ex}}
\end{figure}

\begin{example}
\label{ex:wadf-simple}
In Figure~\ref{fig:wadf-simple-ex}a a simple wADF with three arguments is shown that are intended to decide the acceptance status of a paper based on that paper's significance and scientific methodology.
On the right side of that figure (b), a value ordering $\Vu$ is shown, with the intended meaning that $\tvu$ denotes no knowledge w.r.t.\ the arguments, tendencies denote a certain leaning towards acceptance or rejection, and the information maximal values denote acceptance, borderline acceptance, or rejection.
\end{example}

As for ADFs, we will
use propositional formulas $\varphi_s$ interpreted over $V$ to specify acceptance conditions.
The understanding is that a formula $\varphi_s$ specifies a function $C_s$ such that for each interpretation \mbox{$w:\parents{s}\to V$}, $C_s(w)$ is obtained by considering $w(\varphi_s)$, the evaluation of the formula $\varphi_s$ under the interpretation $w$.
Unlike in classical propositional logic, there is no single standard way of interpreting formulas in the multi-valued case.
Thus the user (specifying the wADF) should state how formulas are to be evaluated under interpretations of atoms by values from $V$.

\begin{example}
Continuing Example~\ref{ex:wadf-simple},
let us define acceptance conditions for each argument. Say argument ``paper's significance'' (shortened to $s$) shall be set to ``accept'' and that the paper's ``scientific methodology'' (shortened to $m$) shall be set to ``borderline'' (e.g.\ because of peer reviewing).
This can be expressed simply by stating: $\varphi_s = $ accept, $\varphi_m = $ borderline. The third argument, shortened to $a$, shall depend on the status of the two other arguments, namely 
by taking the most informative value ``compatible'' (w.r.t.\ information ordering) to the values given to the other arguments. That is, take the usual greatest lower bound for two values w.r.t.\ the information ordering shown in Figure~\ref{fig:wadf-simple-ex}b.
Say we formalize this by $\varphi_a = s \land m$, defining the conjunction as the meet.
\end{example}

In case the truth values in $V$ are $\ileq$-incomparable, the information ordering on the truth values \mbox{$\Vu=V\cup\set{\tvu}$} can be defined analogously to the ordering for standard ADFs (where \mbox{$V=\two$} with \mbox{$\tvt\not\ileq\tvf$} and \mbox{$\tvf\not\ileq\tvt$}).
\begin{definition}
  \label{def:flat}
  Let $V$ be a set of truth values with \mbox{$\tvu\notin V$}.
  A relation \mbox{${\ileq}\subseteq \Vu\times\Vu$} is \define{flat} iff for all \mbox{$x,y\in\Vu$}\/:
  \begin{gather*}
    x\ileq y
    \quad\text{iff}\quad
    x=\tvu \text{ or } x=y
  \end{gather*}
  Likewise, a wADF \mbox{$(S,L,C,V,\ileq)$} is \define{flat} iff $\ileq$ is flat.
\end{definition}
As mentioned above, clearly all standard ADFs are flat.
For flat orderings, the greatest lower bound of a subset \mbox{$X\subseteq\Vu$} is obtained thus\/:
\begin{gather*}
  \icap X =
  \begin{cases}
    x & \text{if } X=\set{x} \\
    \tvu & \text{otherwise}
  \end{cases}
\end{gather*}

We now define the semantics. A semantics $\sigma$ takes a wADF $D$ over $V$ and produces a collection $\sigma(D)$ of partial interpretations from $S$ to $V$, that is, functions \mbox{$v: S \to \Vu$} with \mbox{$\Vu = V\cup\set{\tvu}$} where $\tvu$ represents the fact that the value of a certain node is undefined. Given that for standard ADFs the interpretations of interest are partial functions from $S$ to $\two$ (or, equivalently, functions from $S$ to $\three$), this is the obvious generalization we need.\footnote{This differs from approaches like \cite{AmgoudB17} which consider weight assignments as part of the input and is more in line with research in multi-valued logics.}
Let $D=(S,L,C,V,\ileq)$ be a wADF over $V$.
As for standard ADFs, the characteristic operator for $D$
takes a partial interpretation $v$ and produces a new interpretation, $\Gamma_D(v)$.
The new partial interpretation collects information from and mediates between all \emph{completions of $v$}. As in the standard case a completion of $v$  is any total interpretation $w$ that extends $v$ with respect to the information ordering:\footnote{The notion of completion in the original AAAI paper considered total interpretations obtained by replacing $\tvu$ with arbitrary values from $V$ only. This turns out to be insufficient for some of our results to hold.}
\begin{definition}\label{def:completion}
Let $D=(S,L,C,V,\ileq)$ be a wADF over $V$, $v:S\to \Vu$ a partial interpretation. A total interpretation
$w:S\to V$ is a \define{completion} of $v$ if $v\ileq w$.
\end{definition}
The set of all completions of $v$ is denoted by \mbox{$[v]_c$}.

\begin{example}
Continuing Example~\ref{ex:wadf-simple}, say we have interpretation
$v: \{a \mapsto$ ``tendency accept'', $s \mapsto $ ``accept'', $m \mapsto $ ``borderline'' $\}$. Then there are three completions in $v$, $v_1$, $v_2 \in [v]_c$:
$v_1: \{a \mapsto$ ``accept'', $s \mapsto $ ``accept'', $m \mapsto $ ``borderline'' $\}$ and
$v_2: \{a \mapsto$ ``borderline'', $s \mapsto $ ``accept'', $m \mapsto $ ``borderline'' $\}$. The reason is that both arguments $s$ and $m$ already have information maximal values associated with them (any completion assigns thus the same value), but argument $a$ has a non-maximal value, thus any completion $w$ of $v$ assigns to $a$ any value s.t.\ $v(a) \ileq w(a)$. In the example, this means that $a$ is then assigned ``tendency accept'' (same as for $v$), ``accept'' ($v_1$), and ``borderline'' ($v_2$).
\end{example}

Formally, the operator is defined as follows:
for each \mbox{$s \in S$}, the truth value $\Gamma_D(v)(s)$ is the greatest lower bound \wrt \mbox{$(\Vu,\ileq)$} (the consensus) of the set
\mbox{$\set{ C_s(w) \guard w \in [v]_c }$}.
With these specifications the rest is entirely analogous to the definitions for standard ADFs.

\begin{definition}  Let
$D=(S,L,C,V,\ileq)$ be a wADF and
  \mbox{$v: S \to \Vu$}.   Applying $\Gamma_D$ to $v$ yields a new interpretation (the consensus over $[v]_c$) defined as
  \begin{gather*}
    \Gamma_D(v):S\to\Vu
    \quad\text{with}\quad
    s \mapsto \icap \set{ C_s(w) \guard w\in[v]_c }.
  \end{gather*}
\end{definition}

As usual, we can now define the semantics via fixpoints.

\begin{definition}
  An interpretation $v$ of a wADF $D=(S,L,C,V,\ileq)$ is
\begin{itemize}

  \item \define{grounded} for $D$ iff \mbox{$v=\lfp(\Gamma_D)$}, \ie $v$ is the least fixpoint of $\Gamma_D$. \\
    Intuition: $v$ collects all the information which is beyond any doubt.

  \item \define{admissible} for $D$ iff $v \leq_i \Gamma_D(v)$. \\
    Intuition: $v$ does not contain unjustifiable information.

  \item \define{preferred} for $D$ iff it is $\leq_i$-maximal admissible for $D$. \\
    Intuition: $v$ has maximal information content without giving up admissibility.
  \item  \define{complete} for  $D$ iff $v=\Gamma_D(v)$. \\
    Intuition: $v$ contains exactly the justifiable information.
\item a \define{model} of $D$ iff $v(s) \neq \tvu$ for all $s \in S$ and $\Gamma_D(v) = v$. \\
    Intuition: $v$ contains exactly the information that is justifiable when
each statement has a defined truth value.
  \end{itemize}
\end{definition}
Again we use $\adm(D)$, $\com(D)$ and $\prf(D)$ to denote the set of all admissible, complete and preferred interpretations for $D$, respectively.
Moreover, $\md(D)$ gives the set of all models of $D$.

\begin{example}
Continuing Example~\ref{ex:wadf-simple}, let us consider some interpretations.
As usual, $I_\tvu = \{a \mapsto \tvu, s \mapsto \tvu, m \mapsto \tvu\}$ is admissible.
Since $\Gamma_D(w)(s) = $ ``accept'' and $\Gamma_D(w)(m) = $ ``borderline'', for any interpretation $w$ (due to definition of their acceptance conditions), it holds that it is admissible to assign these two arguments any value lower or equal to these values (and argument $a$ ``tendency accept'', ``no tendency'', or $\tvu$).
There is only one complete interpretation:
$v = \{a \mapsto $ ``tendency accept'', $s \mapsto $ ``accept'', $m \mapsto$ ``borderline'' $\}$.
This implies that $v$ is the (unique in this case) model of this wADF, and also the only preferred interpretation, as well as the grounded interpretation.
\end{example}

Independently of the previous example, we want to emphasize that we have to show existence of the least fixpoint of $\Gamma_D$, otherwise the grounded interpretation is not well-defined in general.
The simplest way to do this is to show monotonicity of the operator $\Gamma_D$.

To this end, we lift the information ordering on $\Vu$ point-wise to interpretations over $\Vu$.
For $v,w:S\to\Vu$, we~set
\begin{gather*}
  v \ileq w
  \quad\text{iff}\quad
  \forall s\in S: v(s)\ileq w(s)
\end{gather*}
The pair $(\set{v:S\to\Vu},\ileq)$ then forms a CPO in which the characteristic operator $\Gamma_D$ of wADFs is monotone.
The least element of this CPO is the interpretation \mbox{$\vu: s\mapsto \tvu$} that maps every statement to $\tvu$ and we will
also use $\icup X$ to denote the least upper bound for subset $X$ of $(\set{v:S\to\Vu},\ileq)$.

\begin{proposition}
  \label{thm:monotone}
  The operator $\Gamma_D$ is $\ileq$-monotone, that is:
  for all interpretations \mbox{$v,w:S\to\Vu$} we have that \mbox{$v \leq_i w$} implies \mbox{$\Gamma_D(v) \leq_i \Gamma_D(w)$}.
\end{proposition}
\begin{proof}
    Let \mbox{$v,w: S \to \Vu$} be two interpretations such that \mbox{$v \leq_i w$}.
    By definition of $[\cdot]_c$, we find that \mbox{$[w]_c \subseteq [v]_c$}.
    Thus also $\set{C_s(u) \guard u\in [w]_c}\subseteq \set{ C_s(u) \guard u\in [v]_c }$ for any \mbox{$s\in S$}.
    It follows that \mbox{$\icap\set{ C_s(u) \guard u\in [v]_c } \ileq \icap\set{C_s(u) \guard u\in [w]_c}$}, that is, \mbox{$\Gamma_D(v)(s)\ileq \Gamma_D(w)(s)$} for any \mbox{$s\in S$}.
\end{proof}

Existence of the least fixpoint of $\Gamma_D$ then follows via the fixpoint theorem for monotone operators in complete partial orders (see, e.g., \cite{davey-priestley}, Theorem~8.22).

The following result is a generalization of Theorem~25 of \cite{Dung1995} and Theorem~1 by \cite{BrewkaESWW2013}.
\begin{theorem}
  \label{thm:complete-semilattice}
  Let $D$ be a weighted {\sc adf} with an information ordering $\ileq$.
  \begin{enumerate}
  \item Each preferred interpretation for $D$ is complete, but not vice versa.
  \item The grounded interpretation for $D$ is the $\ileq$-least complete interpretation.
  \item The complete interpretations for $D$ form a complete meet-semilattice with respect to $\ileq$.
  \end{enumerate}
  \end{theorem}
\begin{proof}
The first item is shown analogous to a previous result in~\cite[{{Theorem 3.10}}]{Strass2013c}. More concretely, the first statement in the theorem holds for every approximating operator, such as $\Gamma_D$.
The second item follows directly from definition (the grounded interpretation is a fixed point, as are all the complete interpretations; the grounded interpretation is the least one).
The proof of the third item follows the same line of reasoning as the proof of the third item of~\cite[{{Theorem 1}}]{BrewkaESWW2013}.
\end{proof}

Next, we show
that the well-known relationships between Dung semantics carry over to our generalizations.

\begin{theorem}
  \label{thm:relationships}
  Let $D$ be a weighted {\sc adf}.
  It holds that
  $$\md(D),\prf(D) \subseteq \com(D) \subseteq \adm(D).$$
  If $D$ is flat, then additionally $\md(D)\subseteq \prf(D)$.
\end{theorem}
\begin{proof}
The ``non-flat'' case follows from definitions and from Theorem~\ref{thm:complete-semilattice}: every preferred interpretation is complete, every complete interpretation is admissible. Every model is a complete interpretation by definition.
For the flat case, every value assigned to a statement by a model is, directly, information maximal (due to the flat ordering). Thus, a model is an information maximal complete interpretation, which directly implies that a model is a preferred interpretation under the assumption that the information ordering is flat.
\end{proof}

The proviso that \mbox{$D=(S,L,C,V,\ileq)$} be flat is necessary for the inclusion \mbox{$\md(D)\subseteq\prf(D)$}:
consider \mbox{$S=\set{a}$} with \mbox{$L=\set{(a,a)}$} and $C_a$ given by $w\mapsto w(a)$ (that is, \mbox{$\varphi_a=a$});
now if there are \mbox{$x,y\in V$} with \mbox{$x\ilt y$}, then we find that \mbox{$v=\set{a\mapsto x}$} is a model that is not preferred.

The flatness property is also crucial in the following result
that lets us compute the grounded semantics by iterative application of the characteristic operator.
\begin{proposition}
\label{prop:iterativeGroundedFlat}
Let $D=(S,L,C,V,\ileq)$ be a wADF such that
every ascending chain in $(\set{v:S\to\Vu},\ileq)$ is finite.
Then,
there is some $n\in\naturals$ such that
$\Gamma_D^n(\vu)$ is the grounded interpretation of $D$. Note that $\Gamma_D^i$ denotes $i$ iterative applications of $\Gamma_D$.
\end{proposition}
\begin{proof} As every ascending sequence
$v_1 \ilt v_2 \ilt \dots$ of interpretations is finite
and since $\Gamma_D$ is monotone, it follows from
a known result (\cite{davey-priestley}, Theorem~8.8(2)) that $\Gamma_D$ is Scott-continuous.
Then, it follows from the Kleene Fixed-Point Theorem that $v=\icup \set{\Gamma^i(\vu) \guard i\in\naturals}$ is the grounded interpretation of $D$.
As the chain $\Gamma_D(\vu) \ilt \Gamma_D(\Gamma_D(\vu)) \ilt \dots$ is finite, we get that $v=\Gamma^n(\vu)$ for some $n\in\naturals$.
\end{proof}
Note that this result applies to important families of wADFs such as all flat wADFs $D=(S,L,C,V,\ileq)$ where $S$ is finite or all wADFs $D=(S,L,C,V,\ileq)$ with finite $V$.

In general, one cannot guarantee that we reach the grounded interpretation by iterative application of the characteristic operator.
\begin{example}
Consider wADF \mbox{$D=(\set{s},\set{(s,s)},\set{C_s},\unit,\leq)$}, where the acceptance condition of the single statement $s$
is defined as
\[
C_s(v) = \begin{cases}
  v(s)+\frac{0,5-v(s)}{2} &\text{for } v(s)<0.5,\\
  1 &\text{otherwise.}
\end{cases}
\]
We then have
$$
\begin{array}{r@{}l}
\Gamma_D(\vu)(s) &= 0.25, \\
\Gamma_D^2(\vu)(s) &= 0.375, \\
\Gamma_D^3(\vu)(s) &= 0.4375, \\
\Gamma_D^4(\vu)(s) &= 0.46875, \\
\dots,
\end{array}
$$
that is, for iterative application of $\Gamma_D$ on $\vu$, the values converge towards $0.5$.
However, the interpretation $v=\set{s\mapsto 0.5}$ is not the grounded interpretation of $D$ as it is no fixed point of $\Gamma_D$.
In fact, $\Gamma_D(v)(s)=1$.
Note that every wADF is guaranteed to have a grounded interpretation, in this case it is $v_g=\set{s\mapsto 1}$.
\end{example}

Despite we do not always get the grounded interpretation by iterative application of the characteristic operator,
if we do reach a fixed point this way then it is the grounded interpretation.
\begin{proposition}
\label{prop:iterativeGrounded}
Let $D$ be a wADF.
If
$v=\icup \set{\Gamma^n(\vu) \guard n\in\naturals}$
is a fixed point of $\Gamma^{V}_D$,
then $v$ is the grounded interpretation of $D$.
\end{proposition}
\begin{proof} The proposition follows from a known result (\cite{davey-priestley}, Theorem~8.15(i)).
\end{proof}

Another result concerns acyclic wADFs, i.e.\ ADFs $(S,L,C,V,\ileq)$ where the pair $(S,L)$ forms an acyclic
directed graph and generalizes a recent result of \cite{Atefeh}.

\begin{theorem}
\label{thm:acyclic-wadf}
For any acyclic wADF $D$ with countable $S$, $\com(D)=\prf(D)=\{v\}$ with $v$ the grounded interpretation of $D$.
\end{theorem}
\begin{proof} We show that an acyclic wADF possesses exactly one complete interpretation. The result then follows from Theorem~\ref{thm:complete-semilattice}.
Given any acyclic wADF $D=(S,L,C,V,\ileq)$, we can (partly) order the statements according to their ``depth'' in the wADF, starting from statements without parents.
A path from a statement $s_1$ to a statement $s_n$, in $D$, is defined as $p = (s_1,s_2,...,s_{n-1},s_n)$ with each $s_i \in S$ and for all $s_i$, $s_{i+1}$ ($1 \leq i \leq n-1$) we have $s_i \in \parents{s_{i+1}}$.
The length of $p$ (denoted as $|p|$) is $n$.
We define the depth $d(s)$ of a statement $s\in S$ in $D$ as
$d(s)=max\{ |p| \mid p=(s',...,s)\mbox{ is a path in }D,\parents{s'} = \emptyset\}$.
Towards a contradiction assume $v_1,v_2\in\com(D)$ such that $v_1\neq v_2$.
Let $s$ be a statement with $v_1(s)\neq v_2(s)$ such that $d(s)$ is minimal (i.e., $v_1(s')= v_2(s')$ for all $s'\in S$ with $d(s')<d(s)$).

Consider the case $d(s)=1$, where $s$ is a leaf.
As $\parents{s} = \emptyset$, $C_s(w)=v$ maps to the same value $v$ for every interpretation $w$. Therefore, we have
$$\icap \set{ C_s(w) \guard w\in[v_1]_c}=\icap \set{C_s(w) \guard w\in[v_2]_c}$$
which is a contradiction as this is the same as $\Gamma_D(v_1)(s)=\Gamma_D(v_2)(s)$ which amounts to $v_1(s)=v_2(s)$. Now assume $d(s)>1$.
For every parent $p\in\parents{s}$ we have $d(p)<d(s)$ and
thus, by choice of $s$, it holds that $v_1(p)=v_2(p)$.
When we restrict the completions of $v_1$ and $v_2$ to $\parents{s}$, being the domain relevant for evaluating $C_s$, we get two identical sets of functions $[v_1]'_c=[v_2]'_c$, where
$$[v_1]'_c= \set{s\mapsto w(s) \guard s\in\parents{s}, w\in[v_1]_c}$$ and
$$[v_2]'_c= \set{s\mapsto w(s) \guard s\in\parents{s}, w\in[v_2]_c}.$$
One can see that
$$\icap \set{ C_s(w) \guard w\in[v_1]_c}=\icap \set{ C_s(w) \guard w\in[v_1]'_c}$$ and
$$\icap \set{ C_s(w) \guard w\in[v_2]_c}=\icap \set{ C_s(w) \guard w\in[v_2]'_c}.$$
We end up in the contradiction $\icap \set{ C_s(w) \guard w\in[v_1]_c}=\icap \set{C_s(w) \guard w\in[v_2]_c}$ like in the previous case.
\end{proof}

In the rest of this section we show how stable semantics can be generalized to weighted ADFs. The basic idea underlying stable semantics is to treat truth values asymmetrically. For standard ADFs where only $\tvt$ and $\tvf$ can appear in models, $\tvf$ (false) can be assumed to hold (by default), whereas $\tvt$ (true) needs to be justified by a derivation. Technically this is achieved by building the reduct of an ADF and then checking whether the grounded interpretation of the reduct coincides with the original model on the nodes which ``survive'' in the reduct.

Moving from the two-valued to the multi-valued case offers an additional degree of freedom: it is not clear a priori what the assumed, respectively derived truth values are. The stable semantics we introduce here will thus be parameterized by a subset $W$ of the set of values $V$ over which the weighted ADF is defined.

\begin{definition}
Let $D=(S,L,C,V,\ileq)$ be a wADF. Let $v: S \rightarrow V$ be a model of $D$ (that is, $v$ is total). Let $W \subseteq V$ be the set of assumed truth values. The $v,W$-reduct of $D$ is the wADF $D^v_W = (S^v_W, L^v_W, C^v_W, V, \ileq)$ where
\begin{itemize}
\item $S^v_W = \{s \in S \mid v(s) \notin W\}$,
\item $L^v_W = L \cap (S^v_W \times S^v_W)$,
\item $C^v_W = (C'_s)_{s \in S^v_W}$ where $C'_s$ is obtained from $C_s$ by fixing the value of each parent $s^*$ of $s$ in $D$ such that  $s^* \notin S^v_W$ to $v(s^*)$.
\end{itemize}
\end{definition}
Note that whenever the acceptance function is represented using propositional formulas, the new acceptance function is simply obtained by replacing atoms not in $S^v_W$ by their $v$-values.
Now stable models can be defined as usual:

\begin{definition}
Let $D=(S,L,C,V,\ileq)$ be a wADF and let $v: S \rightarrow V$ be a model of $D$. Let $v_g$ be the grounded interpretation of the $v,W$-reduct of $D$. $v$ is a $W$-stable model of $D$ iff $v(s) = v_g(s)$ for each $s \in S^v_W$.
\end{definition}
This clearly generalizes stable semantics for standard ADFs: just let $V = \{\tvf,\tvt\}$ and $W = \{\tvf\}$. We conclude this section by showing the exact relationship between ADFs and wADFs.

\begin{theorem}
  \label{thm:generalization}
  Let \mbox{$F=(S,L,C)$} be an ADF. The wADF associated to $F$ is $D_F=(S,L,C,\{\tvt,\tvf\},\ileq)$ with $\ileq$ as defined in the background section.
An interpretation $v$ is a model/admissible/complete/preferred/grounded for $F$ iff
  it is a model/ admissible/complete/preferred/grounded for $D_F$. Moreover, $v$ is stable for $F$ iff it is $\{\tvf\}$-stable for $D_F$.
\end{theorem}
\begin{proof} It suffices to show that $\Gamma^{\{\tvt,\tvf\}}_{D_F} = \Gamma_{F}$, i.e., that the characteristic operator for the weighted ADF $D_F$ is the same as the characteristic operator for the non-weighted (standard) ADF $F = (S,L,C)$.
To see this, note that $D_F$ recovers all ingredients of standard ADFs: $\ileq$ is a flat ordering, the same as the ``pre-defined'' one for $F$ on $\{\tvt,\tvf\}$, evaluation of acceptance conditions (or formulas if the representation is different) is the same, and truth values are the same, as well.
For the stable semantics, note that the same notion of reduct is recovered for $\{\tvf\}$-stable semantics.
\end{proof}

\section{Weighted ADFs Over the Unit Interval}\label{sec:unit}

In this section we focus on weighted ADFs over the unit interval, that is, wADFs over \mbox{$V = \unit$}.

We first assume a flat information ordering $\ileq$.
As discussed in \Cref{sec:framework}, we will use propositional formulas $\varphi_s$ to specify acceptance conditions over $\unit$.
In the subsequent examples, we employ a formula evaluation that is defined via structural induction as follows:
$$
\begin{array}{c@{\ }l}
w(\varphi\land\psi) &= \min\set{ w(\varphi), w(\psi) }\mbox{,}\\
w(\varphi\lor\psi) &= \max\set{ w(\varphi), w(\psi) }\mbox{, and}\\
w(\neg\varphi) &= 1-w(\varphi)\mbox{.}
\end{array}
$$
(Clearly, a richer formula syntax and other evaluations known from multi-valued logics are possible, but not the main topic of this paper.)
We furthermore allow (representations of) elements of $V$ to appear as atoms in the propositional formulas $\varphi_s$ and let \mbox{$w(a)= a$} for \mbox{$a \in V$}.
This enables us to fix acceptance degrees for specific nodes, and to express upper and lower bounds.
For instance, the formula $\phi \land 0.7$ expresses that the acceptance degree of a node cannot be higher than 0.7, and similarly $\phi \vee 0.7$ expresses that the acceptance degree cannot be below 0.7.

\begin{example}\label{ex:unit}
Consider wADF $D$ over $\unit$ depicted below.
\begin{center}
  \begin{tikzpicture}[node distance=6mm]
    \node[arg] (a) at (0,0) {$a$};

    \node[arg] (b) at (4,0) {$b$};

    \node[arg] (c) at (2,0) {$c$};

    \node[arg] (d)  at (6,0) {$d$};

        \node (va) [below of=a] {$0.8$};
    \node (vb) [below of=b] {$\neg b$};
    \node (vc) [below of=c] {$a\wedge b$};
    \node (vd) [below of=d] {$\neg b \vee 0.6$};
    \path[link] (a) edge[] (c);
    \path[link] (b) edge[] (c);
    \path[link] (b) edge[loop above] (b);
    \path[link] (b) edge[] (d);
  \end{tikzpicture}
\end{center}
Intuitively, $\varphi_a$ fixes the value of $a$ to 0.8.
$\varphi_b$ expresses self-attack. $\varphi_c$ means $c$ is accepted to the extent $a$ and $b$ are, while
$d$ is attacked by $b$. In addition, it is known, for whatever reason, that the value of $d$ must be at least 0.6.

We represent a partial interpretation $v$ as the tuple $(v(a),v(b),v(c),v(d))$.
By Proposition~\ref{prop:iterativeGroundedFlat}, the grounded interpretation can be obtained by iterating $\Gamma_D$ on the interpretation $(\tvu,\tvu,\tvu,\tvu)$. We obtain the least fixpoint $v_1 = (0.8,\tvu,\tvu,\tvu)$. The (unique) model of $D$ is $v_2 = (0.8,0.5,0.5,0.6)$. This model is $W$-stable for $W = \lbrack 0,0.5\rbrack$. To see this, consider the $v_2,W$-reduct which consists of nodes $a$ and $d$ with reduced acceptance conditions $0.8$ and $0.5 \vee 0.6$, respectively. The grounded interpretation assigns to these nodes exactly the values they have in $v_2$, namely $0.8$ and $0.6$. Note that $v_2$ is not $\lbrack 0,0.5\lbrack$-stable. In this case the reduct is identical to $D$ and the grounded interpretation $v_1$ of $D$ differs from $v_2$.

Interpretations $v_1$ and $v_2$ are also the only complete ones.
An interpretation $v$ is admissible for $D$ if and only if
\begin{enumerate}
\item $v(a) = \tvu$ or $v(a) = 0.8$,
\item $v(b) = \tvu$ or $v(b) = 0.5$, \item $v(c) = \tvu$ (if $v(a) = \tvu$ or $v(b) = \tvu$) or \\ $v(c) = 0.5$ (if $v(a) = 0.8$ and $v(b) = 0.5$), and \item $v(d) = \tvu$ (if $v(b) = \tvu$) or $v(c) = 0.6$ (if $v(b) = 0.5$).
\end{enumerate}
The single preferred interpretation for $D$ is $v_2$.
\hfill$\diamond$
\end{example}

The following lemma will be useful to derive results for wADFs on the unit interval (with flat ordering).
\begin{lemma}
\label{lemma:eval-unit}
Let $\varphi_s$ be a propositional formula over variables $S$ (without constants).
Assume that the formulas are evaluated on values in the unit interval as specified in the beginning of Section~\ref{sec:unit}.
Further, let $v:S \rightarrow \{\tvt,\tvf\}$ be a total interpretation assigning variables in $S$ to either true or false, and
$w: S \rightarrow [0,1]$ be a total interpretation assigning variables in $S$ to a value in $[0,1]$ s.t.\ $v(s)=\tvt$ implies $w(s) \geq 0.5$ and $v(s)=\tvf$ implies $w(s) \leq 0.5$.
If $v$ satisfies $\varphi_s$ then $\varphi_s(w) \geq 0.5$ and if $v$ does not satisfy $\varphi_s$ then $\varphi_s(w) \leq 0.5$.
\end{lemma}
\begin{proof}
We prove this lemma by structural induction. For the base cases, i.e., atoms and negated atoms, the implications immediately follow.
For compound formulas, i.e., $A \circ B$ with $\circ \in \{\land,\lor\}$ and $\neg A$, if $v$ satisfies $A \lor B$, then either $v(A)$ or $v(B)$ is $\tvt$, and thus $\max\set{ w(A), w(B) } \geq 0.5$. Similarly for $\land$. For negated formulas, if $v$ satisfies $\neg A$, then $v$ does not satisfy $A$.
\end{proof}

We next explore the relation between wADFs over the unit interval and ADFs. For the proposition we identify the ADF truth values $\tvt$ and $\tvf$ with 1 and 0, respectively. We have the following result:

\begin{proposition}
\label{prop:unit-interval-ordering}
  Let \mbox{$D=(S,L,C,V,\ileq)$} be a wADF with no constants other than $0$ or $1$ appearing in any acceptance formula of $C$,
  and
  \mbox{$D'=(S,L,C,\{0,1\},\ileq')$} be its classical version
	(with $\ileq'\,=\,\ileq\! \cap\, (\{0,1,\tvu\}\times\{0,1,\tvu\})$).
  Furthermore, let $v$ be a partial interpretation assigning truth values in \mbox{$\set{0, 1, \tvu}$} only, and \mbox{$s\in S$}.
  \begin{itemize}
  \item If $\Gamma_D(v)(s) \in \set{0,1}$, then $\Gamma_{D'}(v)(s)\in\set{0,1}$.
  \item If $\Gamma_{D'}(v)(s) = \tvu$, then $\Gamma_D(v)(s)=\tvu$.
  \end{itemize}
\end{proposition}
\begin{proof} Suppose the first item does not hold: $\Gamma_D(v)(s) =1$ and $\Gamma_D'(v)(s)\not\in\set{0,1}$ (symmetric for the case of $0$).
It holds that $\Gamma_D'(v)(s) = \tvu$, and thus, there exists a completion of $v$ to $v'$ (on $\{\tvt,\tvf\}$) s.t.\ $v'$ does not satisfy $\varphi_s$.
Construct an interpretation $w'$ s.t.\ $w'(s) = 1$ iff $v'(s) = \tvt$ and $w'(s) = 0$ iff $v'(s) = \tvf$. We have $w'$ is a completion of $v$ (on the unit interval).
By Lemma~\ref{lemma:eval-unit}, it holds that $\varphi_s(w')\leq 0.5$. This is a contradiction to $\Gamma_D(v)(s) =1$ (all completions on the unit interval have to evaluate to $1$ for $\varphi_s$).

The second item of the proposition can be shown analogously. Suppose the contrary, then there exist two completions of $v$ (on $\{\tvt,\tvf\}$), s.t.\ one satisfies $\varphi_s$ while the other does not. By Lemma~\ref{lemma:eval-unit}, it holds that the
corresponding interpretations on the unit interval (identifying, again, $0$ with $\tvf$ and $1$ with $\tvt$), evaluate once to a value at most $0.5$ and otherwise to a value at least $0.5$. To see that the end result cannot be $0.5$, by definition of the connectives, only values from the interpretation (here only $\{0,1\}$) can be re-produced when evaluating the connectives. Therefore, the characteristic operator $\Gamma_D(v)(s)$ returns $\tvu$.
\end{proof}
This result cannot be strengthened, in particular $\Gamma_{D'}(v)(s)$ may be 1 or 0, yet $\Gamma_D(v)(s) = \tvu$, as illustrated as follows\/:
\begin{example}\label{ex:taut} Consider the graph consisting of nodes $a$, $b$ with acceptance formulas
  \mbox{$\varphi_a= a$} and \mbox{$\varphi_b = a \lor \neg a$}.
  It is easy to see that in the weighted case the grounded interpretation is $\set{ a \mapsto \tvu, b \mapsto \tvu}$.
  In contrast, the standard approach yields the grounded interpretation $\set{ a \mapsto \tvu, b \mapsto 1}$.
  This is due to the fact that \mbox{$a \lor \neg a$} is a tautology in two-valued logic, but not when the unit interval and the above specified formula evaluation is used.
  Here the formula may have any value \mbox{$x\geq 0.5$}.
\hfill$\diamond$
\end{example}

A possible remedy would be to define non-standard interpretations $\land^*,\lor^*$ for the connectives $\land,\lor$, for example\/:
\begin{itemize}
\item $x \land^* y = 1$ if $x > 0.5$ and $y > 0.5$; $x \land^* y = 0$ otherwise
\item $x \vee^* y = 1$ if $x > 0.5$ or $y > 0.5$; $x\lor^* y=0$ otherwise
\end{itemize}
However, this approach appears to throw out the baby with the bath water: the different behavior of wADFs in such examples stems from the fact that they make more fine-grained distinctions. It is thus not unintended. Note also that there is an easy alternative option to specify tautological acceptance conditions for wADFs: simply replace $a \vee \neg a$ with $1$.

So far we have used a flat information ordering with least element $\tvu$ and different elements in $\unit$ incomparable.
Of course, nothing prevents us from choosing a more refined ordering, for instance the ordering $\ileq'$ given by
\begin{gather*}
  x \ileq' y
  \quad\text{iff}\quad
  x\ileq y \text{ or }
  y < x \leq 0.5 \text{ or }
  0.5 \leq x < y
  \end{gather*}
That is, $0.5$ is immediately above $\tvu$, and a value smaller than 0.5 is more informative if it is closer to 0, a value greater than 0.5 is more informative if it is closer to 1.
The pair $(\unit\cup\set{\tvu},\ileq')$ again forms a complete partial order;
for any non-empty \mbox{$X\subseteq\unit\cup\set{\tvu}$}, its greatest lower bound is given by
\begin{gather*}
  \textstyle\icap' X =
  \begin{cases}
    \tvu & \text{if } \tvu\in X \\
    \min X & \text{if } X\subseteq [0.5,1] \\
    \max X & \text{if } X\subseteq [0,0.5] \\
    0.5 & \text{otherwise.}
  \end{cases}
\end{gather*}
The CPO property extends to the pointwise extension of $\ileq'$ to valuations. Thus for any \mbox{$D=(S,L,C,\unit,\ileq')$}, its characteristic wADF operator $\Gamma_D$ is well-defined, in particular its least fixpoint $\lfp\!\left(\Gamma_D\right)$ exists and is uniquely determined.

\begin{example}
  \label{ex:unitrefined}
  Consider again Example \ref{ex:unit}, but this time with information ordering $\ileq'$.  Again we iterate $\Gamma_D$ on the interpretation $(\tvu,\tvu,\tvu,\tvu)$.
  With the extended information ordering we obtain the fixpoint $v = (0.8,0.5,0.5,0.6)$ which is guaranteed to be the grounded interpretation by Proposition~\ref{prop:iterativeGrounded}.
  The admissible semantics is given by all interpretations in
  $$
  \begin{array}{rl}
  \{ \set{a \mapsto \tv{v_a}, b \mapsto \tv{v_b} , c \mapsto \tv{v_c},d \mapsto \tv{v_d}} \mid & \tv{v_a}\in\set{\tvu}\cup [0.5;0.8],\\
                                                                                               & \tv{v_b}\in\set{\tvu,0.5},\\
                                                                                               & \tv{v_c}\in\set{\tvu,0.5},\\
                                                                                               & \tv{v_d}\in\set{\tvu}\cup [0.5;0.6]\}.
  \end{array}
  $$
  As every complete interpretation is admissible and $v$ is $\ileq'$-minimal among all complete interpretations but, in this case, $v$ coincides with the $\ileq'$-maximal admissible interpretation,
  we get that $v$ is the only complete and preferred interpretation as well as the only model.
  Similar as in Example \ref{ex:unit}, the $v,W$-reduct for $W = \lbrack 0,0.5\rbrack$ consists of nodes $a$ and $d$ with reduced acceptance conditions $0.8$ and $0.5 \vee 0.6$, respectively.
  Thus, $v$ is a $W$-stable model. Unlike in Example \ref{ex:unit}, $v$ is also $W'$-stable, for $W' = \lbrack 0,0.5\lbrack$, since it is the grounded interpretation and the ADF coincides with its $v,W'$-reduct.
\hfill$\diamond$
\end{example}

\section{Alternative Valuation Structures}
\label{sec:alternative}
In the last section we considered values in $\unit$. Of course, there are many more options which have been studied intensively in the area of multi-valued logics (an excellent overview was given by \cite{Gottwald15}), e.g.
    \begin{itemize}
      \item $W_m = \{\frac{k}{m-1} \mid 0 \leq k \leq m-1\}$,       \item Belnap's 4-valued system with $\{\emptyset, \{\bot\},\{\top\},\{\bot, \top\}\}$, or
     \item values from an interval $[a,b]$ within the unit interval ($0 \leq a \leq b \leq 1$).
    \end{itemize}

 Independently from the actual choice, as before, given \mbox{$D=(S,L,C,V,\ileq)$},
for any chosen set of truth degrees $V$, an interpretation assigns a value from $\Vu$ to nodes in $S$, and the characteristic operator $\Gamma_D$ works on partial $V$-interpretations, that is, on functions of the type \mbox{$S \to \Vu$} where \mbox{$\tvu\notin V$}.
Recall that acceptance conditions are of the form \mbox{$C_s: (\parents{s} \to V) \to V$}.

For one choice of values, we represent acceptance conditions as propositional formulas, this time interpreted over $V$.
To make ADF techniques work we need to define the evaluation of propositional connectives over $V$, thus specifying which acceptance condition a formula actually represents, and the information ordering, as done for $\unit$ in \Cref{sec:unit}.

The literature on multi-valued logics provides a rich source of alternative valuation structures with different benefits and properties. It also offers a wide range of options regarding different evaluations of propositional formulas (\eg G\"odel, \L{}ukasiewicz, etc.).
The only general constraint is that the information ordering $\ileq$ on \mbox{$V\cup\set{\tvu}$} must form a complete partial order (CPO) with least element $\tvu$.

In the following we illustrate the use of alternative valuation structures using three different examples.

\myparagraph{$\mathbf{W_3}$:}
For the family of values $W_m$, let us exemplify \mbox{$W_3=\set{0,0.5,1}$}.
We define formula evaluation as before, that is, 0, 0.5, and 1 evaluate to themselves, $\land$ to $\min$, $\lor$ to $\max$, and $\neg y$ to $1-y$.
We choose as the information ordering the smallest reflexive relation containing \mbox{$\tvu \ileq x$} for all \mbox{$x \in W_3$}, where the remaining values are incomparable if they are different.
That is, we (again) utilize here a flat ordering.
This fully specifies wADFs based on $W_3$.
Note that $0.5 \neq \tvu$.
This makes perfect sense as saying ``the acceptance degree is 0.5'' is different from saying ``the acceptance degree is unknown''.

\myparagraph{Belnap:}
We show how the truth degrees in Belnap's four-valued logic can be used in wADFs. The truth degrees are
$$\textbf{B} = \{\emptyset, \{\bot\},\{\top\},\{\bot, \top\}\}.$$
As to formula evaluation we use the standard definitions for Belnap's logic: conjunction/disjunction are the infimum/supremum under the truth ordering $\leq_t$; negation swaps $\{\bot\}$ and $\{\top\}$, and leaves the other two values unchanged. The truth ordering is the reflexive closure of\/:
$$  \{\bot\}  \leq_t  \emptyset \leq_t \{\top\} \hspace{1cm}  \{\bot\}  \leq_t  \{\bot, \top\} \leq_t \{\top\}$$
The information ordering is the reflexive closure of\/:
$$ \tvu \ileq \emptyset \ileq \{\top\} \ileq \{\bot, \top\}  \hspace{0.6cm} \tvu  \ileq \emptyset \ileq \{\bot\} \ileq \{\bot, \top\} $$
With these definitions, the operator $\Gamma_D$ and thus the 4-valued wADF system is fully specified. Note again that \mbox{$\tvu \neq \emptyset$};  treating them as identical would yield a different system.

\myparagraph{Intervals:}
Our approach can also handle intervals. Let us illustrate this idea using intervals from within the unit interval as truth degrees, that is we consider truth values in
 $$\INT = \{ \lbrack a,b \rbrack \mid 0 \leq a \leq b \leq 1\}.$$
Formula evaluation can be defined as follows:
      \begin{align*}
      \lbrack a,b \rbrack \land^* \lbrack c,d \rbrack &= \lbrack min(a,c), min(b,d) \rbrack \\
       \lbrack a,b \rbrack \vee^* \lbrack c,d \rbrack &= \lbrack max(a,c), max(b,d) \rbrack \\
      \neg^* \lbrack a,b \rbrack &= \lbrack 1-b,1-a \rbrack
       \end{align*}
The information ordering is the $\subseteq$-least relation satisfying (1) $\tvu \ileq v$ for all $v \in \INT$, and (2)
$$\lbrack a,b \rbrack \ileq \lbrack c,d \rbrack \quad\text{whenever}\quad \lbrack c,d \rbrack \subseteq \lbrack a,b \rbrack, $$
   which fully defines the characteristic operator $\Gamma_D$. 

\section{Computational Complexity}\label{sec:complexity}

In this section we give a preliminary complexity analysis of weighted ADFs.
For this we assume that all acceptance conditions are specified via propositional formulas, which, additionally, may contain constants from a pre-specified and fixed set of values $V$. Further, we assume that evaluating an acceptance condition under an interpretation that assigns no statement to $\tvu$ can be done in polynomial time, and that comparing two values under the, again pre-specified and fixed, information ordering can, likewise, be computed in polynomial time \wrt the size of the given wADF.
That is, both $V$ and $\ileq$ are fixed ($\vert V \vert$ is a constant) and not part of the input for the problems we study in this section.

We first show that, under fixed and finite valuation structures $V$, complexity upper bounds of wADFs remain the same as for classical ADFs~\cite{StrassW2015} for several central computational tasks. As for ADFs, a cornerstone auxiliary complexity result is the following for checking whether a given interpretation is admissible in a given wADF. 

\begin{proposition}
\label{prop:ver-finite-v}
Verifying whether a given interpretation is admissible in wADFs with fixed and finite valuation structures is in \coNP{}.
\end{proposition}
\begin{proof}
We show the membership result for the complementary problem.
A given interpretation $v$ over set of values $V$ is not admissible in a given wADF $D = (S,L,C,V,\ileq)$ iff
$v \not \ileq \Gamma_D(v)$ iff
$\exists s \in S$ s.t.\ $v(s) \not \ileq \Gamma_D(v)(s)$ iff
$v(s) \not \ileq \icap \set{ C_s(w) \guard w\in[v]_c }$.
It holds that $\set{ C_s(w) \guard w\in[v]_c } \subseteq V$ and there exists a subset $W \subseteq [v]_c$ of completions of $v$ s.t.\
$\set{ C_s(w) \guard w\in[v]_c } = \set{ C_s(w) \guard w \in W } $ and $\vert W \vert \leq \vert V \vert$ (we only need one completion per value). By definition we have $\vert V \vert$ is constant, and therefore $\vert W \vert $ is bounded by a constant.

Non-deterministically guess an $s \in S$, $\vert V \vert$-many completions of $v$, say as set $X$, and compute $\icap \set{ C_s(w) \guard w \in X }$.
It holds that $\set{ C_s(w) \guard w \in X } \subseteq \set{ C_s(w) \guard w\in[v]_c }$.
Now, compute the greatest lower bound $y$ wrt the completions $X$, which acts as an approximation of the greatest lower bound of all completions.
The greatest lower bound $y =\icap \set{ C_s(w) \guard w \in X }$ can be computed in polynomial time, since cardinality of $X$ is bounded by a constant, by assumption evaluation $C_s(w)$ can be computed in polynomial time, and one can check for each value in $V$ whether this value is the greatest lower bound of $\set{ C_s(w) \guard w \in X }$.
It holds that $y$ approximates $\icap \set{ C_s(w) \guard w\in[v]_c }$ by: $\icap \set{ C_s(w) \guard w\in[v]_c } \ileq y$.
The last inequality follows since $y$ is the greatest lower bound of a subset of $\set{ C_s(w) \guard w\in[v]_c }$, and the information ordering and values form a CPO; thus $y$ cannot be lower than the greatest lower bound of the superset, and not incomparable, which would violate uniqueness of the greatest lower bound in the superset.
Having computed $y$, we check whether $v(s) \ileq y$. If $v(s) \not \ileq y$, then $v(s) \not \ileq \icap \set{ C_s(w) \guard w\in[v]_c }$ ($v$ is not admissible). On the other hand, if $v(s) \ileq \icap \set{ C_s(w) \guard w\in[v]_c }$, then $v(s) \ileq y$. Finally, the actual $\icap \set{ C_s(w) \guard w\in[v]_c }$ can always be found by guessing the corresponding completions.

Thus, the problem of checking whether $v$ is not admissible is in \NP, and, therefore, the problem of verifying whether $v$ is admissible in $D$ is in \coNP.
\end{proof}

Based on the previous result, checking whether there exists an admissible interpretation for a given wADF that assigns a given value, different to $\tvu$, to a given statement, has the same complexity upper bound as the analogous task of credulous acceptance on ADFs. The complexity class $\SigmaP{2}$ contains all problems solvable via a non-deterministic polynomial-time algorithm that has access to an \NP{}-oracle (which can solve problems in \NP{} in constant time).

\begin{proposition}
\label{prop:cred-finite-v}
Checking whether there is an admissible interpretation assigning a given statement a given value for wADFs with fixed and finite valuation structures is in $\SigmaP{2}$.
\end{proposition}
\begin{proof}
For given wADF $D=(S,L,C,V,\ileq)$, $a \in S$, and $x \in V$, non-de\-ter\-min\-is\-ti\-cal\-ly guess a partial interpretation $v$ with $v(a) = x$ and check whether $v$ is admissible in $D$ (a problem in \coNP{} due to Proposition~\ref{prop:ver-finite-v}).
\end{proof}

Checking whether all preferred interpretations for a given wADF assign a given value, different to $\tvu$, to a given statement, has the same complexity upper bound as the analogous task of skeptical acceptance on ADFs. The class $\PiP{3}$ is the complement class of $\SigmaP{3}$ (contains the complements of all problems in $\SigmaP{3}$), which in turn contains all problems solvable via a non-deterministic polytime algorithm with access to a $\SigmaP{2}$-oracle.

\begin{proposition}
\label{prop:skept-finite-v}
Checking if all preferred interpretations assign a given statement a given value for wADFs with fixed and finite valuation structures is in $\PiP{3}$.
\end{proposition}
\begin{proof}
Let $D=(S,L,C,V,\ileq)$ be a wADF, $a \in S$, and $x \in V$. To check whether a given interpretation $v$ is preferred in $D$, consider the complementary problem: checking whether $v$ is not preferred in $D$. Non-deterministically guess $v'$ with $v \ileq v'$, and, via an \NP-oracle, both check whether $v'$ is admissible in $D$ and whether $v$ is admissible in $D$. If $v'$ is admissible or $v$ is not admissible, we can directly infer that $v$ is not preferred in $D$. Thus, checking whether $v$ is preferred in $D$ is in $\PiP{2}$.

To check whether all preferred interpretations of $D$ assign $s$ to $x$, consider again the complementary problem: to check whether there exists a preferred interpretation assigning $s$ to a value different than $x$. Non-deterministically guess a partial interpretation $v$ with $v(s) \not = x$ and check whether $v$ is preferred in $D$ via a $\PiP{2}$-oracle. Thus, the problem of verifying whether all preferred interpretations of $D$ assign $s$ to $x$ is in $\PiP{3}$.
\end{proof}

Analogously to ADFs, the same complexity upper bound for existence of stable models can be derived for wADFs with fixed and finite valuation structures. The assumed values can be arbitrarily chosen among the fixed $V$, but are given as input, as well.

\begin{proposition}
\label{prop:existence-stb}
Checking existence of stable models is in $\SigmaP{2}$ for wADFs with fixed and finite valuation structures.
\end{proposition}
\begin{proof}
Assume a given wADF $D=(S,L,C,V,\ileq)$, with $V$ fixed.
We show a polynomial non-deterministic algorithm, with access to an $\NP$ oracle, that decides whether $D$ has a $W$-stable model.
First, non-deterministically construct an interpretation $v$.
Observe that computing the result of $\Gamma^V_R(v')(s)$ can be done in polynomial time with access to an $\NP$ oracle for any given $v'$: checking whether $\Gamma^V_R(v')(s) \not = a$, $a \in V$, is in $\NP$ (non-deterministically guess a completion and evaluate the acceptance condition).
To check whether $v$ is a model of $D$ it is sufficient to check whether $\Gamma_D(v) = v$ and whether $v$ assigns no argument to $\tvu$. By the observation before, polynomially many $\NP$ calls are sufficient.

Next, consider the $v,W$-reduct $R$. It is immediate that the arguments and links of this reduct can be constructed in polynomial time. For the acceptance conditions, it is sufficient to restrict these to the changed parent relation (completions are fixed for non-parents), and, thus, acceptance conditions need not be modified.
To verify stability of $v$, it suffices to compute the grounded interpretation of $R$. To achieve this, iterate over each statement $s$ and compute $\Gamma^V_R(v')(s)$, with $v'$ being the interpretation assigning each statement to $\tvu$.
Applying this procedure for each statement until a fixed point is reached can, likewise, be done in polynomial time (with access to an $\NP$ oracle). Comparing the grounded interpretation of $R$ to $v$ can be done in polynomial time.
\end{proof}

Complexity lower bounds depend on the fixed $V$, information ordering, and formula evaluation. Non-weighted ADFs (\ie \mbox{$V=\two$}) are an example where, for the corresponding fixed components, the complexity lower bounds match the previously shown upper bounds.

Finally, we show a result for wADFs over the unit interval. As before, we make the same assumptions on the acceptance conditions, but let $V$ be the unit interval and assume a flat information ordering ($\tvu$ is strictly lower than all other elements, and all other elements are incomparable), and assume formula evaluation as defined in \Cref{sec:unit}.
Although ADFs are not a special case, \wrt all semantics, of such wADFs, \coNP{}-hardness of verifying whether a given interpretation is admissible follows from a similar argument as for classical (non-weighted) ADFs.

\begin{proposition}
\label{prop:ver-unit}
Verifying whether a given interpretation is admissible in wADFs over the unit interval with flat information ordering is \coNP{}-hard.
\end{proposition}
\begin{proof}
We reduce the problem of checking whether a given propositional formula $\psi$  over variables $B$ is a tautology. W.l.o.g., we assume that $\psi$ does not contain any constants (i.e.\ does not contain $\top$ or $\bot$). Construct wADF $D=(B \cup \{t\},C)$ with
\begin{itemize}
 \item $\varphi_b = b$ with $b \in B$; and
 \item $\varphi_t = \psi \land x$ with $0 < x < 0.5$ (arbitrary choice).
\end{itemize}
The wADF $D$ can be constructed in polynomial time.
We claim that $v$ is admissible in $D$ with $v(b) = \tvu$ and $v(t) = x$ iff $\psi$ is a tautology.
Assume $\psi$ is a tautology. Then any total \emph{two-valued} interpretation (assigning only true or false) satisfies $\psi$. Thus, due to Lemma~\ref{lemma:eval-unit}, for any total interpretation $w: B \cup \{t\} \rightarrow [0,1]$ we have $w(\psi) \geq 0.5$, and, in turn, $w(\varphi_t) = x$ ($\min\set{x,w(\psi)}$ always equals $x$). This means $v$ is admissible in $D$.

Assume $\psi$ is not a tautology. Then there exists a two-valued interpretation $w$ s.t.\ $w(\psi) = \tvf$. We show that $v'$ with
$$v'(b) = \begin{cases}
          1 \text{ if } w(b) = \tvt\\
          0 \text{ if } w(b) = \tvf
         \end{cases}$$
for $b \in B$ and $v'(t) = x$ is a completion of $v$, i.e., $v' \in [v]_c$, and $v'(\varphi_t) = 0$. To show that $v' \in [v]_c$, consider that $v(b) \ileq v'(b)$ for all $b \in B$ and $v'(t) = v(t)$.
Since $v'$ assigns either $0$ or $1$ to any statement in $B$ it holds that $v'(\psi)$ is either $0$ or $1$ (all logical connectives are defined via minimum or maximum functions, or $1-y$). Due to Lemma~\ref{lemma:eval-unit}, $v'(\psi)$ cannot be $1$ (since $w$ does not satisfy $\psi$). Thus, $v'(\varphi_t) = \min\set{0,x} = 0$. This means, there exists a completion of $v$, namely $v'$, and $v(t) = x \not = 0 = v'(t)$, and, in turn, that $v$ is not admissible.
\end{proof}

\section{Related Work and Conclusions}\label{sec:conclusions}

In this paper we introduced a framework for defining weighted ADFs, a generalization of ADFs allowing to assign acceptance degrees to arguments.
The framework is fully flexible regarding the choice of acceptance degrees and their associated
information ordering.
We have provided definitions of the main semantics and showed a number of properties together with a
preliminary complexity analysis.

There is quite some work on weighted argumentation frameworks, and even a section entitled ``Weighted ADFs'' in \cite{BrewkaW2010}. In many cases weights of some sort are added to the \emph{links} in the argument graph, not to the nodes. For instance, Brewka and Woltran use weights on links to simplify the definition of acceptance conditions, an idea that has later been extended to the GRAPPA framework \cite{BrewkaW2014}.
In \cite{DunneHMPW11} weights on links are used as an ``inconsistency budget'': attacks may be disregarded as long as the sum of the weights of disregarded attacks remains under some threshold.

Here we focus on papers assigning acceptance degrees to argument nodes.
One such approach is Gabbay's equational theory of argumentation frameworks \cite{Gabbay12}. He allows for values in the unit interval and represents AFs in an equational form, where the equations specify certain constraints value assignments need to satisfy.

There are also probabilistic extensions of AFs, e.g.\ \cite{DungT10,LiON11,Hunter13,HunterT14b,HunterT14}
(for a complexity analysis for probabilistic AFs see \cite{FazzingaFP15}),
and even of ADFs \cite{DoderP2014}. The main idea is to generate several standard AFs (resp.\ ADFs) which represent the possible situations
induced by the probabilities. The latter can be assigned to arguments, attacks, and in case
of ADFs, to acceptance conditions. The evaluation of frameworks generated this way
 follows the standard approach, and the results of these evaluations are aggregated accordingly.
The behavior of the semantics is thus triggered via all relevant subgraphs. A related approach in a multi-valued setting is
\cite{Dondio14,Dondio17}.

Social AFs~\cite{Leite11} extend standard AFs by adding to each argument associated numbers of positive and negative votes.
The semantics describe how the votes propagate
through the network, yielding a non-linear system of equations.
Recently, several properties and semantics for weighted AFs have
been proposed in \cite{Amgoud17a,Amgoud17b}. In those works, weights
for arguments are also given from the unit interval, interpreted
in the sense that the greater the value the more acceptable the argument.
The focus is on the definition of new  semantics dedicated for weighted AFs, rather than  on
generalizing standard semantics.
However, the properties proposed in \cite{Amgoud17b} adapted to our setting are of
interest and thus
are on our agenda
for future work.

Finally in \cite{BesnardH01} acceptance grades of arguments are derived from the structure of the argument tree. The authors attempt to ``\emph{provide an abstraction of an argument tree in the form of a single number}''. In a similar vein
Grossi and Modgil
\cite{GrossiM15}
derive acceptance grades from the structure of the underlying AF, e.g.\ the number of attacks against which a particular argument is defended. These approaches are orthogonal to ours.

Our generalization of ADFs differs significantly from the mentioned papers in at least the following respects:
\begin{enumerate}
\item We are more general than existing work (with the exception of the work of \cite{DoderP2014}) in taking ADFs rather than AFs as starting point.
\item Rather than focusing on a single set of acceptance values, we provide a framework where the values can be freely selected based on the needs of a particular application. \item Our semantics are a direct generalization of the operator-based ADF semantics and does not require the computation and aggregation of results for various subgraphs. Moreover, we obtain reasonable results also in cases where equational approaches do not have solutions.
    \item Finally, the choice of an adequate information ordering allows us to do some fine tuning which is not possible in any approach we are aware of.
\end{enumerate}

As to future work, we first want to explore restricted subclasses of weighted ADFs.
In particular, we would like to exploit the known concept to express
Dung AFs as ADFs (see \cite{BrewkaESWW2013}, Theorem 2) in order to investigate
a new form of weighted AFs as a subclass of weighted ADFs. Our general definitions of the standard semantics for weighted ADFs will readily deliver natural
definitions of semantics for such weighted AFs.
Furthermore, the subclass of \emph{bipolar} ADFs has been recognised as a useful class, as they are strictly more expressive than AFs while of equal computational complexity~\cite{StrassW2015,strass15expressiveness,linsbichler16auniform}, so we intend to investigate \emph{weighted bipolar} ADFs.
A first step would be the generalization of supporting and attacking links to the multi-valued setting, for example via regarding  of acceptance functions' monotonicity and antimonotonicity in single (function) arguments~\cite{baumann-strass17on-the-number}.

We also would like to explore an idea that goes back to \cite{BogaertsVD16}.
In our approach, as in standard AFT, interpretations \mbox{$v:S\to V$} of atoms $S$ with truth values $V$ are approximated by functions \mbox{$v':S\to\Vu$} with \mbox{$\Vu=V\cup\set{\tvu}$} for \mbox{$\tvu\notin V$}.
Such three-valued/partial interpretations consequently represent the set of their completions.
However, not all sets of total interpretations can be represented as completions of a partial interpretation. This is due to the fact that partial interpretations either assign a specific value, or leave the value completely undefined. This suggests the following\/:
  A \emph{generalized partial interpretation} ($\gpi$) of $S$ in $V$ is a total function \mbox{$v:S \to 2^V\setminus\set{\emptyset}$}, that is, a $\gpi$ assigns to each element of $S$ a non-empty subset of the values in $V$.
  In this new setting, the total function \mbox{$w:S \to V$} is a \emph{completion of $v$} if and only if for all \mbox{$s \in S$}, we have \mbox{$w(s) \in v(s)$}.
Based on the notion of a $\gpi$ we can generalize the characteristic ADF operator $\Gamma_D$ to operate on gpis rather than partial interpretations.
For each node $s$, the revised gpi $\Gamma_D(g)$ returns the set of values that are obtained by evaluating the acceptance condition of $s$ under any completion of the input gpi $g$. A further investigation of this topic is on our agenda.

\section*{Acknowledgments}
This work has been supported by the Deutsche Forschungsgemeinschaft (DFG) under grant BR 1817/7-2 and by the Austrian Science Fund (FWF): I2854 and P30168-N31.

\small
\bibliographystyle{plain}

\end{document}